\documentclass{article}

\usepackage[preprint]{neurips_2026}
\usepackage{booktabs}
\usepackage{multirow}

\usepackage[utf8]{inputenc}
\usepackage[T1]{fontenc}
\usepackage[pagebackref=true,breaklinks=true,colorlinks,urlcolor={blue!70!black},linkcolor={red!90!black},citecolor={blue!90!black},bookmarks=false]{hyperref}
\usepackage{url}
\usepackage{amsfonts}
\usepackage{amssymb}
\usepackage{nicefrac}
\usepackage{microtype}
\usepackage{xcolor}
\usepackage{graphicx}
\usepackage{amsthm}
\usepackage{amsmath}
\usepackage{colortbl}

\usepackage{mathtools}
\usepackage{wrapfig}
\usepackage{listings}
\usepackage{tabulary}
\usepackage{etoolbox}
\usepackage{enumitem}
\usepackage{xspace}
\usepackage{algorithmic}
\usepackage{algorithm}
\usepackage{float}
\usepackage{subfig}
\usepackage{fontawesome5}
\usepackage{minitoc}
\newtheorem{Theorem}{Theorem}[section]
\newtheorem{Definition}[Theorem]{Definition}
\newtheorem{Proposition}[Theorem]{Proposition}
\newtheorem{Assumption}[Theorem]{Assumption}

\newtheorem{Corollary}[Theorem]{Corollary}

\theoremstyle{definition}
\newtheorem{Remark}[Theorem]{Remark}

\newcommand{\method}{MARS\xspace}

\title{\method: \underline{M}argin-\underline{A}dversarial \underline{R}isk-controlled \underline{S}topping \\ for Parallel LLM Test-time Scaling}

\author{
  Wenbo Chen\footnotemark[1]$^{\;\,,\spadesuit}$ \quad
  Puheng Li\footnotemark[1]$^{\;\,,\bigstar}$ \quad
  Mengyang Liu\thanks{Equal contribution. Correspondence to \texttt{wenbochen111@outlook.com}.}$^{\;\,,\spadesuit}$ \quad
  Weijie Su$^{\Diamond}$ \quad
  Tianpei Xie$^{\spadesuit}$ \\[4pt]
  $^{\spadesuit}$Amazon\thanks{Work conducted outside of the authors' roles at Amazon.} \qquad $^{\bigstar}$Stanford University \qquad $^{\Diamond}$University of Pennsylvania
}

\usepackage{amsmath,amsfonts,bm}

\def\eqref#1{equation~\ref{#1}}

\def\1{\bm{1}}

\DeclareMathAlphabet{\mathsfit}{\encodingdefault}{\sfdefault}{m}{sl}
\SetMathAlphabet{\mathsfit}{bold}{\encodingdefault}{\sfdefault}{bx}{n}

\begin{document}

\maketitle
\enlargethispage{0.4in}
\vspace{-10pt}

\begin{center}
{\color{black}\faGithub}\ \textbf{Code:}
  \href{https://github.com/Wenbo11/MARS}{\texttt{https://github.com/Wenbo11/MARS}}
\\
\raisebox{-0.1em}{\includegraphics[height=1.1em]{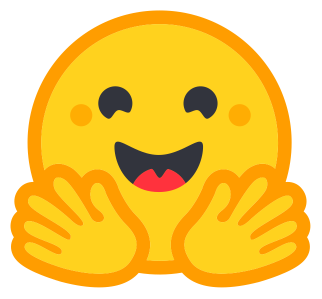}}\ \textbf{Thinking Traces:}
  \href{https://huggingface.co/datasets/wenbochen111/MARS}{\texttt{https://huggingface.co/datasets/wenbochen111/MARS}}
\end{center}
\vspace{-4pt}

\begin{abstract}
Parallel test-time scaling samples many reasoning traces and majority-votes their answers, improving LLM accuracy but requiring traces to run to completion, incurring substantial computational overhead. We observe that probing partial traces at intermediate checkpoints can extract current answers without disrupting generation, revealing an evolving aggregate vote. Based on this observation, we introduce \emph{\method}, a margin-adversarial stopping rule that estimates which active traces are likely to change their answers and stops once the leader remains safe under a conservative bound on future vote movement. The rule separates two sources of uncertainty. It learns the trace-level switch probabilities that determine how much of the current margin is likely to be retained, while handling the harder question of where switching traces land through an adversarial bound calibrated from warmup traces. With true switch probabilities, \method guarantees with high probability that the early-stopped answer matches the full-budget vote. In practice, a five-feature logistic model closely matches oracle switching behavior. Across three reasoning models and three competition-math benchmarks,~\method saves 25--47\% of self-consistency tokens and 14--29\% on top of DeepConf Online, a strong confidence-weighted baseline that already filters and truncates weak traces, while matching the accuracy of the corresponding full-budget baselines.

\end{abstract}

\vspace{-10pt}
\begin{figure}[H]
\centering
\includegraphics[width=0.9\textwidth]{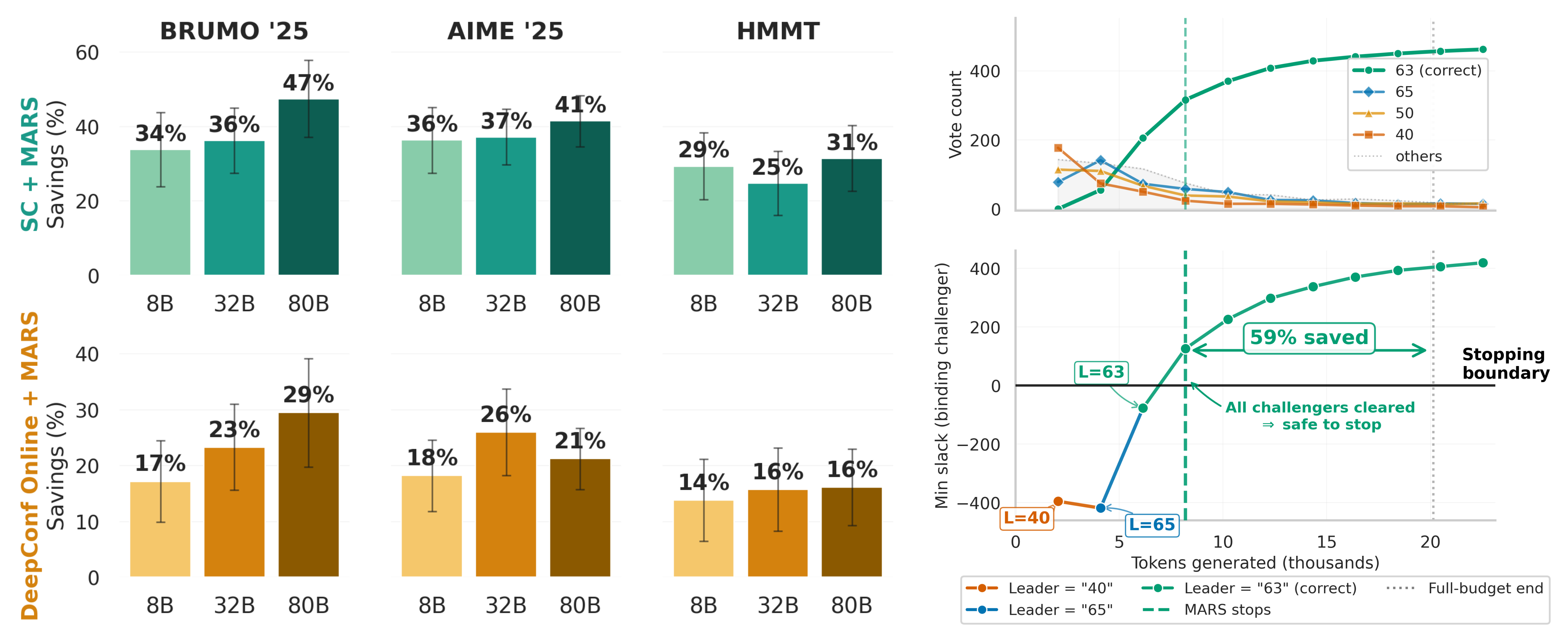}
\vspace{-6pt}
\caption{\textbf{Left:} Token savings achieved by \method across three models (DeepSeek-R1-8B, Qwen3-32B, Qwen3-next-80B) and three competition math benchmarks, under both self-consistency (SC) and DeepConf Online \citep{fu2025deep} voting, while matching the accuracy. Error bars show 95\% CI across questions. \textbf{Right:} \method in action on HMMT Q22 (DeepSeek-R1-8B). Top: vote share evolution across probes. Bottom: minimum slack (binding challenger margin minus adversarial threshold); \method stops when slack crosses zero, keeping the correct answer and saving 59\% of tokens.}
\label{fig:overview}
\end{figure}

\section{Introduction}
\label{sec:intro}

Parallel test-time scaling, sampling many reasoning traces and aggregating their answers by majority vote, is among the most reliable ways to improve LLM accuracy on hard problems \citep{wang2022self,brown2024large,snell2024scaling}. But reliability comes at a cost: a 512-trace run on a single competition-math problem generates millions of tokens. On most questions, the winning answer is decided well before all traces finish. The remaining computation is pure waste.

A key enabler is that modern reasoning models expose long thinking traces \citep{qwen3report}, and recent probing work shows that intermediate answers can be elicited from partial traces during parallel thinking \citep{zheng2026parallel}. This makes \emph{probing} possible: extracting intermediate answers from partial traces mid-generation without disrupting the ongoing reasoning process. Probing transforms parallel decoding from opaque to observable. At each checkpoint we know who is voting for what, how the plurality has shifted, and how many traces remain undecided. Despite its power, probing for early stopping remains largely unexplored, with only one concurrent method \citep{zheng2026parallel} exploiting it in the parallel-vote setting.

\begin{wrapfigure}{r}{0.48\textwidth}
\vspace{-10pt}
\centering
\includegraphics[width=0.48\textwidth]{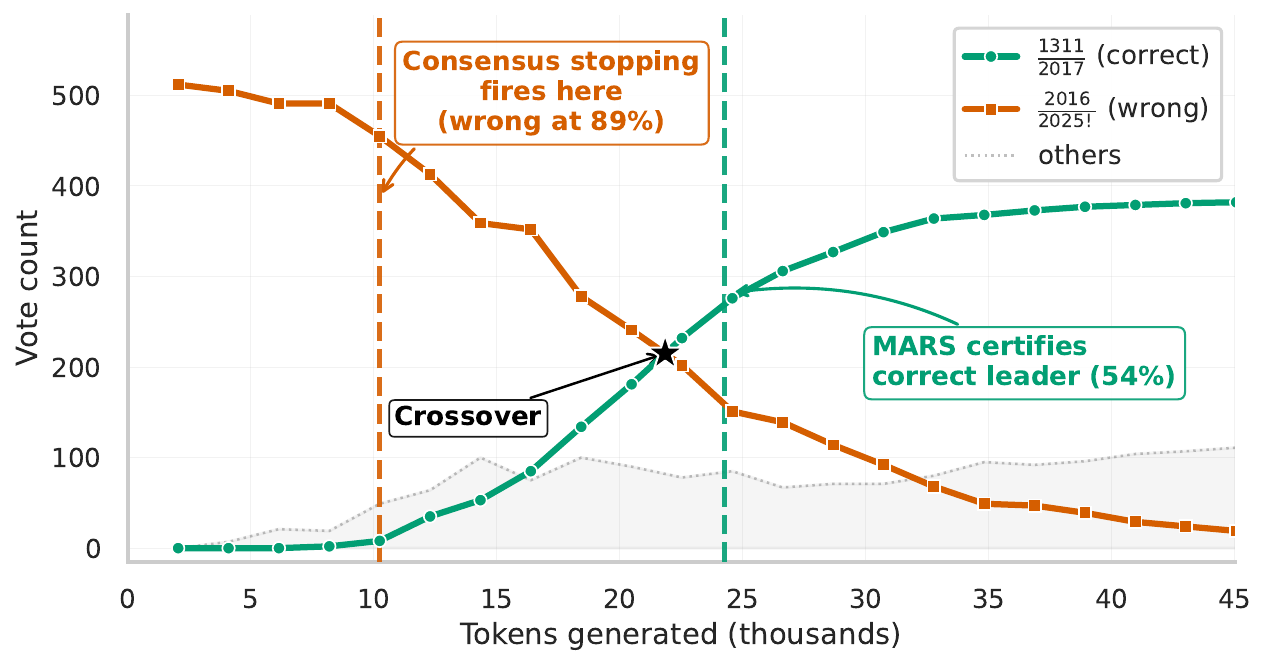}
\caption{Consensus stopping fails on hard questions. On HMMT Q6 (DeepSeek-8B), 89\% of traces initially vote for the wrong answer and the wrong answer leads till the middle of generation. Consensus stopping fires early, locking in the error. \method waits until the correct answer overtakes and certifies it.}
\label{fig:pp_failure}
\vspace{-10pt}
\end{wrapfigure}

But observability is not the same as a stopping criterion. Given a live view of the vote, when is it safe to stop? The naive answer, stop when the majority answer has been stable for several checkpoints, is a consensus heuristic. And consensus heuristics fail precisely when they matter most: on hard problems. Fig.~\ref{fig:pp_failure} illustrates a typical failure mode. At an early checkpoint the majority of traces converge on an \emph{incorrect} preliminary answer. A consensus-based rule sees stability and stops, locking in the wrong answer. With more computation, a smaller group of traces would have switched to the correct answer and eventually overtaken the early leader. The concurrent Parallel-Probe method \citep{zheng2026parallel} instantiates exactly this failure: on HMMT with DeepSeek-8B, consensus stopping drops accuracy from 70\% to 35\%, cutting performance in half while ostensibly ``saving tokens.''

The failure is not accidental. Hard questions are precisely those where the correct answer emerges late, after many traces revise their initial reasoning. Any rule that equates \emph{stability} with \emph{safety} will systematically stop too early on the questions where early stopping is most dangerous. This reveals the core research question:

\begin{center}
\emph{Can we stop parallel generation early, provably preserving the final vote outcome,\\while still saving a substantial fraction of tokens?}
\end{center}

The key insight is that the right object is not whether the vote has changed, but whether it \emph{can} still change. A lead of 300 votes is safe if the remaining uncertain traces cannot plausibly coordinate 150 switches toward a single challenger. The same lead is fragile if most leader-supporting traces are still mid-reasoning and likely to revise. Safety is a \emph{margin} question: the current lead, minus the worst-case damage from future switches, must remain positive for every challenger simultaneously.

We introduce \method (\textbf{M}argin-\textbf{A}dversarial \textbf{R}isk-controlled \textbf{S}topping), a principled early-stopping rule for parallel LLM decoding. At each checkpoint, \method (i) probes every active trace for its current answer, (ii) estimates the probability that each trace will switch before the full-budget endpoint, and (iii) computes the maximum margin loss that future switches could adversarially cause against each challenger. Generation stops only when every challenger is certified: the leader's margin exceeds the adversarial switch damage. With true switch probabilities, this yields a high-probability guarantee that the early-stopped answer matches the full-budget vote.

The modeling problem separates cleanly into two parts. \emph{Whether} a trace will switch is learnable from its own history (checkpoint position, probe confidence, answer-flip count, stability streak) using a lightweight logistic model trained on a handful of warmup traces per question. \emph{Where} a switching trace will land is fundamentally harder: the destination is an open-ended answer string that depends on future reasoning. Rather than learning a full destination distribution, \method treats destinations adversarially and calibrates a per-question contraction parameter from warmup evidence to capture how much uncertain switch mass actually behaves adversarially in practice.

On AIME 2025, HMMT, and BRUMO 2025, across DeepSeek-R1-8B, Qwen3-32B, and Qwen3-next, \method saves 25--47\% of self-consistency tokens and an additional 14--29\% on top of DeepConf Online \citep{fu2025deep}, while matching the full-budget baseline across all 18 settings (within 0.6 percentage points where accuracy changes, and improving it by up to 0.8 points in several settings; Fig.~\ref{fig:overview}). In direct comparison, Parallel-Probe \citep{zheng2026parallel} achieves comparable savings only by sacrificing 9--35 percentage points of accuracy; when tuned to preserve accuracy, its savings collapse to ${\leq}4\%$. The results confirm that principled margin certification, not consensus detection, is the right foundation for safe early stopping in parallel reasoning.

\subsection{Problem setup}
\label{sec:setup}

We formalize the stopping problem for a single prompt. The system launches $N$ reasoning traces in parallel. Each trace would ordinarily run to a full-budget endpoint $T$. We observe the traces at checkpoints $t_1<\cdots<t_P=T$.

\paragraph{Trace states.} At checkpoint $t$, the base sampler reports a status for each trace:
\[
s_j(t)\in\{\mathrm{running},\mathrm{finished},\mathrm{discarded}\}.
\]
A \emph{running} trace is still generating and may change its future answer. A \emph{finished} trace has produced a terminal answer, so its contribution is fixed. A \emph{discarded} trace has been deprecated by the base pipeline (e.g., filtered for low confidence); it contributes zero weight. The active set is $\mathcal{A}_t=\{j:s_j(t)=\mathrm{running}\}$, and $\mathcal{F}_t$ denotes all information available at checkpoint $t$.

\paragraph{Probing.} For a running trace, we extract its current answer by appending a short stop-thinking instruction to the partial generation, collecting the answer, and discarding the probe suffix. The original trace continues unmodified from the checkpoint. This matches the behavior of thinking-mode models that can produce an answer when interrupted \citep{qwen3report}. We denote the extracted answer $a_j(t)$; for finished traces this is the terminal answer, and for discarded traces $a_j(t)=\varnothing$.

\paragraph{Voting.} Each trace carries an effective nonneg weight $w_j\in[0,w_{\max}]$, with $w_j=0$ for discarded traces. The vote for answer $a\neq\varnothing$ at checkpoint $t$ is
\[
V_a(t)=\sum_{j=1}^N w_j\,\mathbf{1}\{a_j(t)=a\},
\]
and the current leader is $L(t)=\arg\max_a V_a(t)$ with deterministic tie-breaking. Uniform weights recover self-consistency; confidence-derived weights recover schemes like DeepConf \citep{fu2025deep}.

\paragraph{Stopping.} We seek a stopping time $\tau\leq T$ that saves generation cost while preserving the full-budget answer with high probability. For a user-specified risk level $\delta\in(0,1)$:
\[
P\bigl(L(\tau)\neq L(T)\bigr)\leq\delta.
\]
This is a guarantee relative to the full-budget vote, not directly to the unknown ground truth. Benchmark accuracy is evaluated separately in \S\ref{sec:experiments}.

\section{Margin-Adversarial Risk-controlled Stopping (\method)}
\label{sec:method}

\begin{figure}[t]
\centering
\includegraphics[width=\textwidth]{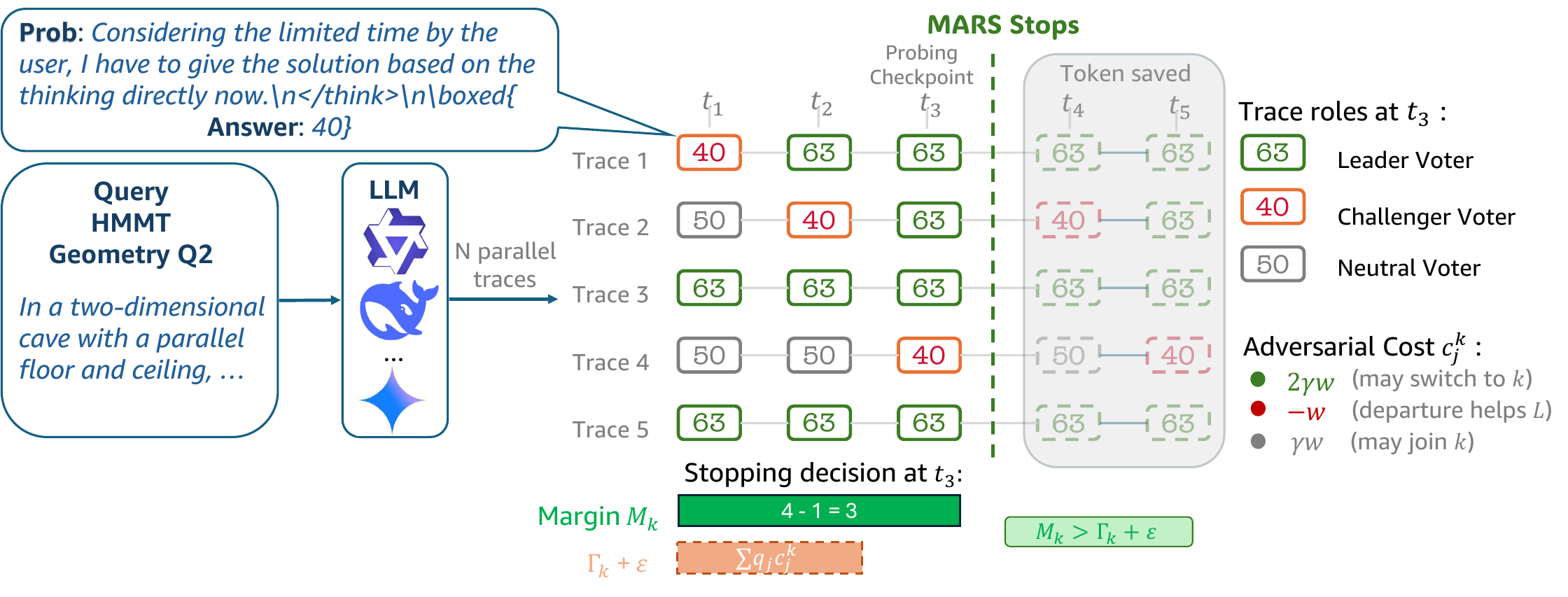}
\caption{Illustration of \method on a single question. At each checkpoint, active traces are probed for current answers. The margin test checks whether the leader's advantage over every challenger exceeds the expected adversarial switch damage. Generation stops only when all challengers are certified.}
\label{fig:pipeline}
\end{figure}


The key idea behind \method is simple: the current leader is safe to output when its margin over every challenger exceeds the worst-case damage that future answer switches could inflict (Fig.~\ref{fig:pipeline}). We first define the margin and present the stopping rule with its safety guarantee (\S\ref{sec:margin_criterion}--\S\ref{sec:stopping_rule}), introduce a calibrated relaxation (\S\ref{sec:gamma}), and describe the practical algorithm (\S\ref{sec:practical}).

\subsection{The margin as stopping criterion}
\label{sec:margin_criterion}

Consider 512 parallel traces, each producing an intermediate answer at a checkpoint. Suppose the current leader has 300 votes and the nearest challenger has 150. Is it safe to stop? The answer depends on how many traces might still change their minds. If only 20 traces are still mid-reasoning, the lead is unassailable. If 400 traces are volatile, the lead is fragile. The relevant quantity is not the margin alone, but the margin relative to remaining uncertainty.

We formalize this intuition. At checkpoint $t$, define the margin of the leader $L=L(t)$ over each challenger $k$ as
\[
M_k(t)=V_L(t)-V_k(t).
\]
If $M_k(T)>0$ for every challenger $k$ at the full-budget endpoint $T$, then the checkpoint leader survives. Conversely, the full-budget winner differs only if some challenger achieves $M_k(T)\leq 0$. The stopping problem reduces to certifying, for every challenger simultaneously, that its margin cannot be closed by future vote changes.

The checked challenger set $\mathcal{K}(t)$ contains every nonleader answer with positive current vote, plus a synthetic unseen challenger $\bot$ with $V_{\bot}(t)=0$. The $\bot$ challenger covers answer strings that are absent at checkpoint $t$ but could appear later: it pessimistically allows all undecided mass to coordinate on a novel answer.

\subsection{The stopping rule}
\label{sec:stopping_rule}

Between checkpoint $t$ and the full budget $T$, some active traces will change their answers. Let $X_j = \mathbf{1}\{a_j(T)\neq a_j(t)\}$ indicate whether trace $j$ switches, and let $q_j(t) = P(X_j=1\mid\mathcal{F}_t)$ be its conditional switch probability, estimated from observable features of the trace (position, confidence, flip history, streak length). Finished and discarded traces have $q_j=0$.

When trace $j$ switches, the damage it inflicts on the margin against challenger $k$ depends on its current vote. We bound this damage adversarially by considering the worst-case destination:

\begin{Definition}[Adversarial switch cost]\label{def:cost}
\begin{equation}\label{eq:cost}
c_j^k =
\begin{cases}
2w_j, & a_j(t)=L,\\
-w_j, & a_j(t)=k,\\
w_j, & a_j(t)\notin\{L,k\}.
\end{cases}
\end{equation}
\end{Definition}

This is the worst case scenario, since: 1)~a leader-voter switching away removes $w_j$ from the leader and at worst adds $w_j$ to challenger $k$, for total damage $2w_j$; 2)~a challenger-voter leaving \emph{helps} the leader regardless of destination, giving $-w_j$; 3)~a neutral voter at worst joins $k$, costing $w_j$. The intuition is simple: a challenger with volatile supporters is less threatening than its vote count suggests.

The expected adversarial damage against challenger $k$ is
\begin{equation}\label{eq:threshold}
\Gamma_k(t)=\sum_{j\in\mathcal{A}_t} q_j\, c_j^k.
\end{equation}
\method stops when the current margin exceeds this expected damage plus a concentration correction for every challenger:
\begin{equation}\label{eq:stopping}
M_k(t)\geq \Gamma_k(t)+\epsilon(N,\delta)
\qquad \text{for all } k\in\mathcal{K}(t).
\end{equation}

\begin{Assumption}[Bounded weights]\label{ass:bounded}
For all traces, $0\leq w_j\leq w_{\max}$.
\end{Assumption}

\begin{Assumption}[Conditional independence]\label{ass:indep}
Conditional on $\mathcal{F}_t$, the switch indicators $\{X_j\}_{j\in\mathcal{A}_t}$ are independent.
\end{Assumption}

\begin{Theorem}[Per-challenger safety]\label{thm:safety}
Under Assumptions~\ref{ass:bounded}--\ref{ass:indep}, using true switch probabilities $q_j(t)$,
\begin{equation}\label{eq:safety}
P\!\left(M_k(T)\leq 0\mid \mathcal{F}_t\right)
\leq
\exp\!\left(
-\frac{\left(M_k(t)-\Gamma_k(t)\right)_+^2}{2w_{\max}^2N_{\mathrm{active}}}
\right),
\end{equation}
where $N_{\mathrm{active}}=\max\{1,|\mathcal{A}_t|\}$.
\end{Theorem}

\begin{Corollary}[Risk-controlled stop]\label{cor:stopping}
Set
\begin{equation}\label{eq:correction}
\epsilon(N,\delta)=w_{\max}\sqrt{2N_{\mathrm{active}}\log\frac{N}{\delta}}.
\end{equation}
If Eq.~\ref{eq:stopping} holds for every $k\in\mathcal{K}(t)$, then $P\bigl(L(T)\neq L(t)\mid \mathcal{F}_t\bigr)\leq \delta$.
\end{Corollary}

The proof applies Hoeffding's inequality to $\sum_j X_j c_j^k$ and takes a union bound over challengers; details are in Appendix~\ref{app:proofs}. The lightweight logistic model that estimates $q_j$ from warmup traces is described in \S\ref{sec:experiments}.

\begin{Remark}[Conservative relaxation of optimal stopping]\label{rem:optimal}
The Bayes-optimal rule stops at the earliest $t$ where $P(L(T)\neq L(t)\mid\mathcal{F}_t)\leq\delta$. Since Corollary~\ref{cor:stopping} provides a \emph{sufficient} condition for this event, the certified \method rule ($\gamma=1$, true $q_j$) is strictly more conservative: it stops at the same time or later, never earlier. The formal derivation via margin decomposition is in Appendix~\ref{app:optimality}.
\end{Remark}

\subsection{Calibrated conservatism via $\gamma$}
\label{sec:gamma}

The fully adversarial cost $c_j^k$ assumes that \emph{all} switching mass coordinates against the same challenger. In practice, switching traces disperse across many answers, and some even join the leader. This makes the rule overly cautious.

We relax the adversarial costs with a contraction parameter $\gamma\in[1/2,1]$:
\begin{equation}\label{eq:gamma_cost}
c_j^k(\gamma)=
\begin{cases}
2\gamma w_j, & a_j(t)=L,\\
-w_j, & a_j(t)=k,\\
\gamma w_j, & a_j(t)\notin\{L,k\}.
\end{cases}
\end{equation}
The lower bound $\gamma\geq 1/2$ preserves the irreducible cost: when a leader-voter switches, it loses at least $w_j$ from the leader regardless of destination. Only the ``which challenger benefits'' portion is contracted. The calibrated stopping criterion becomes
\begin{equation}\label{eq:empirical_stop}
M_k(t)\geq \sum_{j\in\mathcal{A}_t} q_j\, c_j^k(\gamma)
\qquad \text{for all } k\in\mathcal{K}(t).
\end{equation}
Eq.~\ref{eq:empirical_stop} with $\gamma<1$ or estimated $\hat q_j$ is an empirical relaxation; Corollary~\ref{cor:stopping} applies only to the certified variant with $\gamma=1$ and true switch probabilities.

We calibrate $\gamma$ per question from warmup traces run to completion. Let $\gamma_{\mathrm{warmup}}$ be the smallest $\gamma$ that would have stopped correctly on the warmup set (i.e., without changing the warmup full-budget winner). We add an upper-confidence correction:
\begin{equation}\label{eq:ucb}
\gamma_{\mathrm{out}}=\min\!\left(1,\;\gamma_{\mathrm{warmup}}+\frac{z}{\sqrt{n_{\mathrm{elig}}}}\right),
\end{equation}
where $n_{\mathrm{elig}}$ is the number of eligible warmup checkpoints and $z=1.0$ in all experiments.

\subsection{Practical implementation}
\label{sec:practical}

The stopping rule (Eq.~\ref{eq:empirical_stop}) requires two learned quantities per question: the switch probabilities $q_j(t)$ and the contraction $\gamma$. Both are estimated from a small set of 16 warmup traces run to completion before the main sample is stopped. We fit a per-question logistic regression on (trace, checkpoint) pairs, labeled by whether the trace's final answer differs from its checkpoint answer, using five $\mathcal{F}_t$-measurable features: checkpoint position, probe confidence, answer-flip count, streak length, and confidence trend. This gives $\hat q_j(t)$ for every active trace at every checkpoint (details in Appendix~\ref{app:qmodel}). The same warmup traces calibrate $\gamma$ via Eq.~\ref{eq:ucb}.

At each checkpoint, $\mathcal{K}(t)$ includes every nonleader answer with positive votes, plus an unseen challenger $\bot$ with zero current votes. For $\bot$, no trace is a $k$-voter, so there is no $-w_j$ benefit: the cost is $c_j^\bot = 2\gamma w_j$ for leader-voters and $\gamma w_j$ for all others. This guards against novel answers emerging after checkpoint $t$. The complete procedure is given in Algorithm~\ref{alg:stopping}.

\begin{algorithm}[t]
\caption{\method: Margin-Adversarial Early Stopping}
\label{alg:stopping}
\begin{algorithmic}[1]
\STATE \textbf{Input:} trace budget $N$, checkpoints $\{t_1,\ldots,t_P\}$, warmup traces, calibration strength $z$.
\STATE Run warmup traces to completion; 
\STATE Fit logistic regression $\hat q$ on warmup (trace, checkpoint) pairs.
\STATE Calibrate $\gamma$ from warmup via Eq.~\ref{eq:ucb}.
\STATE Launch the $N$ main traces in parallel.
\FOR{checkpoint $t\in\{t_1,\ldots,t_P\}$}
    \STATE Update trace statuses and active set $\mathcal{A}_t$.
    \STATE Probe running traces for current answers $a_j(t)$.
    \STATE Compute votes, leader $L(t)$, margins $M_k(t)$.
    \STATE Estimate $\hat q_j(t)$ for $j\in\mathcal{A}_t$; set $\hat q_j(t)=0$ for inactive traces.
    \STATE Build $\mathcal{K}(t)$: observed challengers $\cup$ unseen challenger $\bot$.
    \STATE \textbf{if} $M_k(t)\geq \sum_{j\in\mathcal{A}_t} \hat q_j\, c_j^k(\gamma)$ for all $k\in\mathcal{K}(t)$ \textbf{then return} $L(t)$.
\ENDFOR
\STATE \textbf{return} the full-budget leader $L(T)$.
\end{algorithmic}
\end{algorithm}

\section{Experiments}
\label{sec:experiments}

\subsection{Setup}\label{subsec:setup}

\paragraph{Datasets and models.}
We evaluate on three competition mathematics benchmarks, a standard stress test for LLM reasoning \citep{hendrycks2021measuring}: AIME 2025, HMMT, and BRUMO 2025, with 30 problems each. We test three reasoning models: DeepSeek-R1-Distill-Qwen-8B~\citep{guo2025deepseek}, Qwen3-32B, and Qwen3-next~\citep{qwen3report}. For each model--dataset pair, we generate a pool of 4{,}096 complete traces per problem and probe intermediate answers every 2{,}048 tokens. This gives 9 model--dataset settings.

\paragraph{Evaluation protocol.}
We simulate 512-trace parallel runs by bootstrap sampling from the 4{,}096-trace pool, using 64 iterations per question. In each bootstrap run, 16 traces are designated as warmup traces for switch-model fitting, $\gamma$ calibration, and DeepConf threshold calibration. The remaining traces form the main parallel sample on which stopping decisions are evaluated. Token savings are reported relative to the corresponding full-budget voting baseline:
\[
\mathrm{savings}=\mathrm{mean}_{\text{questions}}\left(1-\frac{\mathrm{tokens}_{\method}}{\mathrm{tokens}_{\mathrm{baseline}}}\right).
\]
Accuracy is the benchmark accuracy of the answer returned by the stopped or full-budget vote. Token savings include the 16 warmup traces in the denominator; per-checkpoint probe completions (up to 50 tokens each) are excluded from the count as they are small relative to full trace lengths.

\paragraph{Voting pipelines.}
We evaluate \method as an early-stopping layer on top of two voting pipelines. Self-consistency (SC) uses uniform weights $w_j=1$ and majority vote \citep{wang2022self}. DeepConf Online (DCO) \citep{fu2025deep} weights and filters traces by confidence; traces below a warmup-calibrated threshold are discarded. We remove DeepConf’s adaptive sampling component, as it introduces sequential dependencies after the warm-up phase and therefore breaks full parallelism during inference.
\method operates on the resulting effective vote identically in both cases, testing whether aggregate-level margin stopping provides savings beyond per-trace confidence stopping.

\paragraph{Compared methods.}
For each pipeline, we report the fully conservative variant ($\gamma=1$), the calibrated variant (\S\ref{sec:practical}), and an oracle-$q$ diagnostic that replaces the learned switch model with the retrospective indicator $q_j^*=\mathbf{1}\{a_j(t)\neq a_j(T)\}$. Full per-method numbers are in Appendix Tables~\ref{tab:full_ds8b}, \ref{tab:full_q32b}, and \ref{tab:full_qnext}.

\subsection{Main results}

\begin{table}[t]
\centering
\small
\setlength{\tabcolsep}{4pt}
\begin{tabular}{llcccccc}
\toprule
& & \multicolumn{3}{c}{SC Family} & \multicolumn{3}{c}{DeepConf Family} \\
\cmidrule(lr){3-5} \cmidrule(lr){6-8}
Model & Dataset & Baseline & +\method & Savings & Baseline & +\method & Savings \\
\midrule
\multirow{3}{*}{DeepSeek-8B}
& BRUMO '25  & 93.2 & 93.2 & 33.7\% & 93.3 & 93.3 & 17.1\% \\
& AIME '25   & 83.3 & 83.3 & 36.3\% & 88.6 & 88.6 & 18.2\% \\
& HMMT       & 70.0 & 70.0 & 29.3\% & 78.7 & 78.6 & 13.8\% \\
\midrule
\multirow{3}{*}{Qwen3-32B}
& BRUMO '25  & 90.9 & 90.3 & 36.2\% & 92.4 & 93.0 & 23.2\% \\
& AIME '25   & 80.1 & 80.1 & 37.1\% & 79.8 & 80.6 & 26.0\% \\
& HMMT       & 62.8 & 62.8 & 24.7\% & 65.1 & 65.5 & 15.7\% \\
\midrule
\multirow{3}{*}{Qwen3-next}
& BRUMO '25  & 96.7 & 96.7 & 47.4\% & 95.4 & 95.5 & 29.4\% \\
& AIME '25   & 86.9 & 87.7 & 41.4\% & 90.0 & 90.0 & 21.2\% \\
& HMMT       & 86.9 & 86.9 & 31.4\% & 82.2 & 82.0 & 16.1\% \\
\bottomrule
\end{tabular}
\caption{Accuracy (\%) and token savings for calibrated \method. ``Baseline'' is the full-budget voting accuracy; ``+\method'' is the early-stopped accuracy. Across all 18 settings \method preserves accuracy within 0.6pp where it changes and improves it by up to 0.8pp in several settings, while saving 25--47\% (SC) and 14--29\% (DeepConf) of tokens.}
\label{tab:main}
\end{table}

\paragraph{Token savings.}
Table~\ref{tab:main} and Fig.~\ref{fig:overview} show that \method saves 25--47\% of tokens under SC and an additional 14--29\% on top of DeepConf. DeepConf starts from a stronger and cheaper baseline because it already filters and truncates low-confidence traces, so the additional savings from \method are smaller but still consistent. Savings are largest in settings where the model reaches consensus early, such as Qwen3-next on BRUMO.

\paragraph{Accuracy.}
Across all 18 settings, early-stopped accuracy is preserved: where it changes it stays within 0.6 percentage points of the full-budget baseline, and in several settings it improves by up to 0.8 points (e.g., Qwen3-next on AIME under SC, 86.9\% $\rightarrow$ 87.7\%, and Qwen3-32B on BRUMO under DeepConf, 92.4\% $\rightarrow$ 93.0\%). Most settings match exactly; the largest drop is Qwen3-32B on BRUMO under SC (90.9\% $\rightarrow$ 90.3\%).

\subsection{Ablations}

\begin{figure}[t]
\centering
\includegraphics[width=\textwidth]{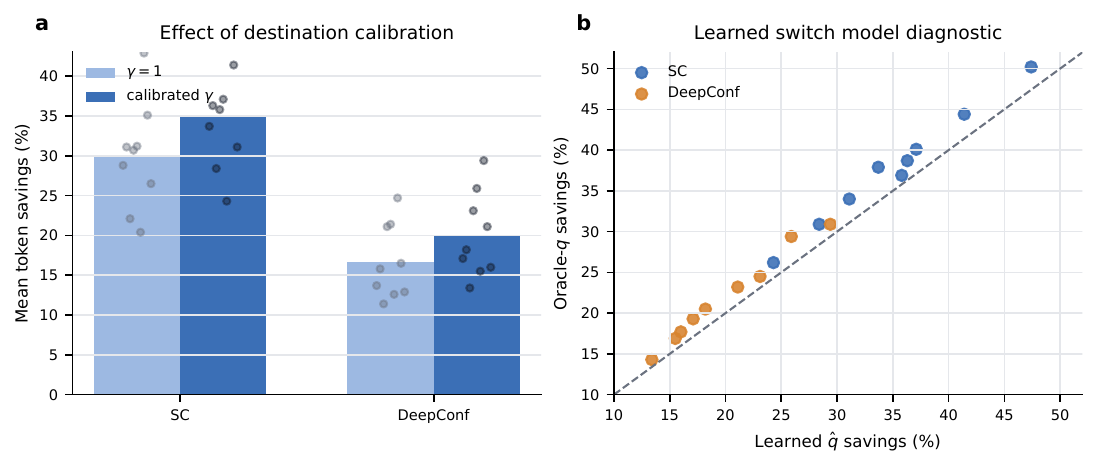}
\caption{Effect of destination calibration ($\gamma$) and switch probability estimation. \textbf{a:} Per-question $\gamma$ calibration adds 4--7pp savings under SC and 2--5pp under DeepConf beyond the fully conservative $\gamma{=}1$ variant, confirming that warmup traces contain useful information about destination behavior. \textbf{b:} The learned 5-feature logistic model closely matches oracle-$q$, with a gap of only 1--4pp, indicating that the switch-event signal is well captured by trace-intrinsic features.}
\label{fig:ablations}
\end{figure}

Fig.~\ref{fig:ablations} summarizes two key ablations across all model--dataset settings.

\paragraph{Destination calibration.}
Even the fully conservative rule ($\gamma=1$) achieves large savings, but per-question $\gamma$ calibration adds 4--7 points under SC and 2--5 points on top of DeepConf. This shows that warmup traces contain useful information about how much switch mass actually behaves adversarially. The calibrated contraction captures part of the destination effect without fitting a full destination model.

\paragraph{Switch probability estimation.}
The learned switch model closely matches oracle-$q$ across all settings, with a consistent gap of only 1--4 percentage points. This confirms that the five trace-intrinsic features capture the switch-event signal well. Together with the $\gamma$ ablation, this supports the two-part design: learn which traces are likely to switch, then use per-question contraction to approximate aggregate destination behavior.

\subsection{Comparison with Parallel-Probe}
\label{sec:pp_comparison}

We compare against Parallel-Probe \citep{zheng2026parallel}, the closest concurrent method in our setting.
Parallel-Probe uses two heuristics: (i) \emph{consensus early stopping}, halt when the majority answer is unchanged for $u$ consecutive probes, and (ii) \emph{deviation-based pruning}, remove any trace disagreeing with consensus for $k$ consecutive probes.
Since the original method probes every 500 tokens while our data probes every 2{,}048 tokens, we run three configurations: a \emph{token-equivalent} mapping ($u{=}3$, $k{=}2$, warmup${=}4$ probes, matching their ${\sim}7{,}000$-token stability window), a conservative variant ($u{=}4$), and their \emph{literal} parameters ($u{=}14$, $k{=}7$, warmup${=}15$).
All Parallel-Probe runs use uniform weighting and no filtering, matching their published setup.

\begin{table}[t]
\centering
\small
\setlength{\tabcolsep}{3.5pt}
\begin{tabular}{llcccccc}
\toprule
& & \multicolumn{2}{c}{PP ($u{=}3$)} & \multicolumn{2}{c}{PP ($u{=}4$)} & \multicolumn{2}{c}{\method (SC)} \\
\cmidrule(lr){3-4} \cmidrule(lr){5-6} \cmidrule(lr){7-8}
Model & Dataset & Acc. & Sav. & Acc. & Sav. & Acc. & Sav. \\
\midrule
\multirow{3}{*}{DeepSeek-8B}
& BRUMO '25  & 48.2 & 42.0 & 51.2 & 38.5 & \textbf{93.2} & 33.7 \\
& AIME '25   & 60.3 & 39.7 & 63.2 & 33.7 & \textbf{83.3} & 36.3 \\
& HMMT       & 34.6 & 51.4 & 37.8 & 47.0 & \textbf{70.0} & 29.3 \\
\midrule
\multirow{3}{*}{Qwen3-32B}
& BRUMO '25  & 72.6 & 21.9 & 72.6 & 19.4 & \textbf{90.9} & 36.2 \\
& AIME '25   & 74.0 & 20.0 & 77.3 & 14.5 & \textbf{80.1} & 37.1 \\
& HMMT       & 40.4 & 33.4 & 41.1 & 27.9 & \textbf{62.8} & 24.7 \\
\midrule
\multirow{3}{*}{Qwen3-next}
& BRUMO '25  & 70.2 & 33.9 & 72.6 & 30.5 & \textbf{96.7} & 47.4 \\
& AIME '25   & 75.6 & 34.0 & 77.3 & 28.4 & \textbf{86.9} & 41.4 \\
& HMMT       & 47.9 & 50.6 & 51.1 & 46.1 & \textbf{86.9} & 31.4 \\
\bottomrule
\end{tabular}
\caption{Comparison with Parallel-Probe (PP). Accuracy (\%) and per-question token savings (\%). PP achieves comparable savings only by sacrificing a lot of accuracy. Under its literal parameters ($u{=}14$), PP preserves accuracy but saves ${\leq}4\%$ tokens (not shown). \method achieves 25--47\% savings without accuracy loss.}
\label{tab:pp_comparison}
\end{table}

\paragraph{The accuracy--savings dilemma.} Table~\ref{tab:pp_comparison} reveals that Parallel-Probe faces a fundamental tradeoff. With token-equivalent parameters ($u{=}3$), it achieves 20--51\% savings but loses 6--45pp accuracy, severely on harder benchmarks (HMMT: 34.6\% vs.\ the 70.0\% baseline on DeepSeek-8B). With their literal parameters ($u{=}14$), accuracy is preserved but savings collapse to ${\leq}4\%$: most traces have fewer than 14 probe positions, so the stopping criterion never fires.
In contrast, \method achieves 25--47\% savings while preserving accuracy in all 9 settings.

\paragraph{Why mode stability fails.} The root cause is that PP's consensus stability condition conflates \emph{having been stable} with \emph{being safe to stop}. Fig.~\ref{fig:pp_failure} illustrates this on a concrete example: at early checkpoints, 89\% of traces vote for the wrong answer and the majority appears stable, so consensus stopping fires. The correct answer only overtakes much later. PP's pruning exacerbates the problem: branches disagreeing with a (potentially incorrect) early consensus are eliminated, destroying the diversity needed for self-correction. \method's margin-based criterion directly quantifies the remaining adversarial risk and fires only when that risk is provably insufficient to overturn the leader.

\section{Related work}
\label{sec:related_work}

\paragraph{Self-consistency and parallel sampling.}
Chain-of-thought and related decomposition prompts improve multi-step reasoning \citep{wei2022chain,kojima2022large,zhou2023least}. Self-consistency \citep{wang2022self} builds on this by sampling multiple reasoning paths and taking a majority vote, extended to universal self-consistency for free-form answers \citep{chen2023universal} and scaled to large budgets \citep{brown2024large}. Related test-time search methods organize reasoning paths more explicitly, for example through tree-structured exploration \citep{yao2023tree}. DeepConf \citep{fu2025deep} uses model-internal confidence for confidence-weighted voting, filtering, and online truncation, substantially improving accuracy on hard questions. These methods all decide \emph{how to vote} but not \emph{when to stop}, as they run all traces to completion. Our framework is orthogonal: it determines when the vote outcome has stabilized, and applies identically on top of any voting scheme, including both uniform and confidence-weighted voting. This contrasts with other test-time scaling approaches that train explicit verifiers or process reward models to evaluate candidate solutions or reasoning steps~\citep{cobbe2021training,uesato2022solving,lightman2023let,zhang2025generative}, as~\method achieves efficiency purely by observing the generation dynamics of the base model.

\paragraph{Adaptive sample budgets.}
A related line of work reduces cost by drawing fewer \emph{complete} samples, connecting LLM self-consistency to classical sequential testing and time-uniform inference~\citep{wald1945sequential,howard2021time}. Adaptive-Consistency \citep{aggarwal2023adaptive} generates samples sequentially and stops drawing more once a Dirichlet-based criterion indicates vote stability; ESC \citep{li2024esc} optimizes a performance-cost tradeoff; CGES \citep{aghazadeh2025cges}, Reliability-Aware Adaptive Self-Consistency \citep{kim2026reliability}, and Optimal Bayesian Stopping \citep{huang2026optimal} use confidence, reliability, or Bayesian posterior estimates; Certified Self-Consistency \citep{cordero2025certified} and CITE \citep{ota2026cite} derive anytime-valid guarantees; and self-calibration methods use model confidence to allocate test-time compute \citep{huang2025efficient}. These methods decide \emph{how many} complete samples to generate, but every sample runs to completion. In the parallel deployment standard for latency-sensitive applications where all $N$ traces are launched simultaneously, these methods cannot reduce wall-clock latency, since they require sequential generation to decide whether to draw the next sample. Our setting is different: all traces run in parallel, and we stop them \emph{mid-generation} by probing intermediate answers at checkpoints, directly reducing both total tokens and wall-clock latency.

\paragraph{Dynamic generation-time efficiency.}
A separate line of research targets inference efficiency by dynamically halting or accelerating computation during the generation process itself. At the token and layer levels, techniques such as speculative decoding~\citep{leviathan2023fast, chen2023accelerating} and confident early-exit mechanisms~\citep{schuster2022confident} reduce wall-clock latency during generation. For reasoning models, recent methods stop or prune individual chains based on intermediate confidence, answer entropy, or learned policies~\citep{yang2025dynamic,laaouach2025halt,hou2025thinkprune}. More recently, this philosophy has been elevated to parallel test-time scaling. The concurrent Parallel-Probe method~\citep{zheng2026parallel} implements this by extracting intermediate answers and halting generation when the aggregate majority vote remains stable for a fixed number of probes. However, relying on consensus stability heuristics risks premature convergence, as a historically stable leader may still possess a fragile margin.~\method addresses this by shifting the focus from historical stability to future risk, providing a rigorous margin-based safety guarantee.

\section{Conclusion}
\label{sec:conclusion}

We introduce~\method, an early-stopping layer for parallel LLM test-time scaling.~\method probes partial traces, estimates which traces are still likely to switch answers, and stops when the current vote margin is large enough to absorb adversarial future switches. With true switch probabilities and full destination conservatism, the rule gives a high-probability guarantee of matching the full-budget vote. Empirically, a lightweight learned switch model closely tracks oracle switch behavior, and per-question contraction calibration captures additional destination structure beyond the conservative worst case. Together these components save 25--47\% of self-consistency tokens and 14--29\% on top of DeepConf while matching the accuracy of full-budget baselines across all tested settings. These results suggest that aggregate margin stability is a distinct and useful source of test-time efficiency.

There are some natural directions for future work. Currently, the probing interval is fixed (every 2{,}048 tokens). An adaptive scheme that probes more frequently when the margin is tight and less when it is large could reduce probing overhead while maintaining responsiveness. In this work, we only consider mathematical reasoning as the evaluation task; applying \method to other domains (code generation, agentic tasks) where trace switch dynamics may differ is an important next step. Another interesting future direction is to explore other methods to accurately estimate where traces will switch in the end. If the distribution of challengers has good statistical properties, it'll be possible to design a more principled estimator than using the calibration coefficient $\gamma$ as a surrogate.

\bibliographystyle{plain}
\bibliography{ref}

@misc{qwen3report,
  title={{Qwen3} Technical Report},
  author={An Yang and Anfeng Li and Baosong Yang and Beichen Zhang and Binyuan Hui and Bo Zheng and Bowen Yu and Chang Gao and Chengen Huang and Chenxu Lv and Chujie Zheng and Dayiheng Liu and Fan Zhou and Fei Huang and Feng Hu and Hao Ge and Haoran Wei and Huan Lin and Jialong Tang and Jian Yang and Jianhong Tu and Jianwei Zhang and Jianxin Yang and Jiaxi Yang and Jing Zhou and Jingren Zhou and Junyang Lin and Kai Dang and Keqin Bao and Kexin Yang and Le Yu and Lianghao Deng and Mei Li and Mingfeng Xue and Mingze Li and Pei Zhang and Peng Wang and Qin Zhu and Rui Men and Ruize Gao and Shixuan Liu and Shuang Luo and Tianhao Li and Tianyi Tang and Wenbiao Yin and Xingzhang Ren and Xinyu Wang and Xinyu Zhang and Xuancheng Ren and Yang Fan and Yang Su and Yichang Zhang and Yinger Zhang and Yu Wan and Yuqiong Liu and Zekun Wang and Zeyu Cui and Zhenru Zhang and Zhipeng Zhou and Zihan Qiu},
  year={2025},
  eprint={2505.09388},
  archivePrefix={arXiv},
  primaryClass={cs.CL}
}

@article{fu2025deep,
  title={Deep think with confidence},
  author={Fu, Yichao and Wang, Xuewei and Tian, Yuandong and Zhao, Jiawei},
  journal={arXiv preprint arXiv:2508.15260},
  year={2025}
}

@inproceedings{wang2022self,
  title={Self-consistency improves chain of thought reasoning in language models},
  author={Wang, Xuezhi and Wei, Jason and Schuurmans, Dale and Le, Quoc and Chi, Ed and Narang, Sharan and Chowdhery, Aakanksha and Zhou, Denny},
  booktitle={International Conference on Learning Representations},
  year={2023}
}

@article{brown2024large,
  title={Large language monkeys: Scaling inference compute with repeated sampling},
  author={Brown, Bradley and Juravsky, Jordan and Ehrlich, Ryan and Clark, Ronald and Le, Quoc V and Ré, Christopher and Mirhoseini, Azalia},
  journal={arXiv preprint arXiv:2407.21787},
  year={2024}
}

@article{snell2024scaling,
  title={Scaling {LLM} test-time compute optimally can be more effective than scaling model parameters},
  author={Snell, Charlie and Lee, Jaehoon and Xu, Kelvin and Kumar, Aviral},
  journal={arXiv preprint arXiv:2408.03314},
  year={2024}
}

@article{chen2023universal,
  title={Universal self-consistency for large language model generation},
  author={Chen, Xinyun and Aksitov, Renat and Alon, Uri and Ren, Jie and Xiao, Kefan and Yin, Pengcheng and Prakash, Sushant and Sutton, Charles and Wang, Xuezhi and Zhou, Denny},
  journal={arXiv preprint arXiv:2311.17311},
  year={2023}
}

@article{wald1945sequential,
  title={Sequential tests of statistical hypotheses},
  author={Wald, Abraham},
  journal={The Annals of Mathematical Statistics},
  volume={16},
  number={2},
  pages={117--186},
  year={1945}
}

@article{howard2021time,
  title={Time-uniform, nonparametric, nonasymptotic confidence sequences},
  author={Howard, Steven R and Ramdas, Aaditya and McAuliffe, Jon and Sekhon, Jasjeet},
  journal={The Annals of Statistics},
  volume={49},
  number={2},
  pages={1055--1080},
  year={2021}
}

@inproceedings{aggarwal2023adaptive,
  title={Let's Sample Step by Step: Adaptive-Consistency for Efficient Reasoning and Coding with {LLM}s},
  author={Aggarwal, Pranjal and Madaan, Aman and Yang, Yiming and Mausam},
  booktitle={Proceedings of the 2023 Conference on Empirical Methods in Natural Language Processing},
  pages={12375--12396},
  year={2023},
  publisher={Association for Computational Linguistics}
}

@inproceedings{li2024esc,
  title={Escape Sky-High Cost: Early-Stopping Self-Consistency for Multi-Step Reasoning},
  author={Li, Yiwei and Yuan, Peiwen and Feng, Shaoxiong and Pan, Boyuan and Wang, Xinglin and Sun, Bin and Wang, Heda and Li, Kan},
  booktitle={International Conference on Learning Representations},
  year={2024}
}

@article{zheng2026parallel,
  title={Parallel-Probe: Towards efficient parallel thinking via {2D} probing},
  author={Zheng, Tong and Huang, Chengsong and Dai, Runpeng and others},
  journal={arXiv preprint arXiv:2602.03845},
  year={2026}
}

@article{aghazadeh2025cges,
  title={{CGES}: Confidence-guided early stopping for efficient and accurate self-consistency},
  author={Aghazadeh, Ehsan and Ghasemi, Ahmad and Beyhaghi, Hedyeh and Pishro-Nik, Hossein},
  journal={arXiv preprint arXiv:2511.02603},
  year={2025}
}

@article{cordero2025certified,
  title={Certified self-consistency: Statistical guarantees and test-time training for reliable reasoning in {LLMs}},
  author={Cordero-Encinar, Paula and Duncan, Andrew B.},
  journal={arXiv preprint arXiv:2510.17472},
  year={2025}
}

@article{huang2026optimal,
  title={Optimal {B}ayesian stopping for efficient inference of consistent {LLM} answers},
  author={Huang, Jingkai and Ma, Will and Zhou, Zhengyuan},
  journal={arXiv preprint arXiv:2602.05395},
  year={2026}
}

@inproceedings{platt1999probabilistic,
  title={Probabilistic outputs for support vector machines and comparisons to regularized likelihood methods},
  author={Platt, John C.},
  booktitle={Advances in Large Margin Classifiers},
  pages={61--74},
  year={1999},
  publisher={MIT Press}
}

@inproceedings{sglang,
 author = {Zheng, Lianmin and Yin, Liangsheng and Xie, Zhiqiang and Sun, Chuyue and Huang, Jeff and Yu, Cody Hao and Cao, Shiyi and Kozyrakis, Christos and Stoica, Ion and Gonzalez, Joseph E. and Barrett, Clark and Sheng, Ying},
 booktitle = {Advances in Neural Information Processing Systems},
 doi = {10.52202/079017-2000},
 editor = {A. Globerson and L. Mackey and D. Belgrave and A. Fan and U. Paquet and J. Tomczak and C. Zhang},
 pages = {62557--62583},
 publisher = {Curran Associates, Inc.},
 title = {SGLang: Efficient Execution of Structured Language Model Programs},
 url = {https://proceedings.neurips.cc/paper_files/paper/2024/file/724be4472168f31ba1c9ac630f15dec8-Paper-Conference.pdf},
 volume = {37},
 year = {2024}
}

@article{wei2022chain,
  title={Chain-of-thought prompting elicits reasoning in large language models},
  author={Wei, Jason and Wang, Xuezhi and Schuurmans, Dale and Bosma, Maarten and Xia, Fei and Chi, Ed and Le, Quoc V and Zhou, Denny and others},
  journal={Advances in neural information processing systems},
  volume={35},
  pages={24824--24837},
  year={2022}
}

@inproceedings{lightman2023let,
  title={Let's Verify Step by Step},
  author={Lightman, Hunter and Kosaraju, Vineet and Burda, Yuri and Edwards, Harrison and Baker, Bowen and Lee, Teddy and Leike, Jan and Schulman, John and Sutskever, Ilya and Cobbe, Karl},
  booktitle={International Conference on Learning Representations},
  year={2024}
}

@article{cobbe2021training,
  title={Training verifiers to solve math word problems},
  author={Cobbe, Karl and Kosaraju, Vineet and Bavarian, Mohammad and Chen, Mark and Jun, Heewoo and Kaiser, Lukasz and Plappert, Matthias and Tworek, Jerry and Hilton, Jacob and Nakano, Reiichiro and others},
  journal={arXiv preprint arXiv:2110.14168},
  year={2021}
}

@inproceedings{leviathan2023fast,
  title={Fast inference from transformers via speculative decoding},
  author={Leviathan, Yaniv and Kalman, Matan and Matias, Yossi},
  booktitle={International Conference on Machine Learning},
  pages={19274--19286},
  year={2023},
  organization={PMLR}
}

@article{kadavath2022language,
  title={Language models (mostly) know what they know},
  author={Kadavath, Saurav and Conerly, Tom and Askell, Amanda and Henighan, Tom and Drain, Dawn and Perez, Ethan and Schiefer, Nicholas and Hatfield-Dodds, Zac and DasSarma, Nova and Tran-Johnson, Eli and others},
  journal={arXiv preprint arXiv:2207.05221},
  year={2022}
}

@article{guo2025deepseek,
  title={{DeepSeek-R1}: Incentivizing Reasoning Capability in {LLMs} via Reinforcement Learning},
  author={{DeepSeek-AI} and Guo, Daya and Yang, Dejian and Zhang, Haowei and Song, Junxiao and Zhang, Ruoyu and Xu, Runxin and Zhu, Qihao and Ma, Shirong and Wang, Peiyi and others},
  journal={arXiv preprint arXiv:2501.12948},
  year={2025}
}

@article{chen2023accelerating,
  title={Accelerating large language model decoding with speculative sampling},
  author={Chen, Charlie and Borgeaud, Sebastian and Irving, Geoffrey and Lespiau, Jean-Baptiste and Sifre, Laurent and Jumper, John},
  journal={arXiv preprint arXiv:2302.01318},
  year={2023}
}

@article{schuster2022confident,
  title={Confident adaptive language modeling},
  author={Schuster, Tal and Fisch, Adam and Gupta, Jai and Dehghani, Mostafa and Bahri, Dara and Tran, Vinh and Tay, Yi and Metzler, Donald},
  journal={Advances in Neural Information Processing Systems},
  volume={35},
  pages={17456--17472},
  year={2022}
}

@inproceedings{kojima2022large,
  title={Large Language Models are Zero-Shot Reasoners},
  author={Kojima, Takeshi and Gu, Shixiang Shane and Reid, Machel and Matsuo, Yutaka and Iwasawa, Yusuke},
  booktitle={Advances in Neural Information Processing Systems},
  volume={35},
  pages={22199--22213},
  year={2022}
}

@inproceedings{zhou2023least,
  title={Least-to-Most Prompting Enables Complex Reasoning in Large Language Models},
  author={Zhou, Denny and Sch{\"a}rli, Nathanael and Hou, Le and Wei, Jason and Scales, Nathan and Wang, Xuezhi and Schuurmans, Dale and Cui, Claire and Bousquet, Olivier and Le, Quoc and Chi, Ed},
  booktitle={International Conference on Learning Representations},
  year={2023}
}

@article{yao2023tree,
  title={Tree of thoughts: Deliberate problem solving with large language models},
  author={Yao, Shunyu and Yu, Dian and Zhao, Jeffrey and Shafran, Izhak and Griffiths, Tom and Cao, Yuan and Narasimhan, Karthik},
  journal={Advances in neural information processing systems},
  volume={36},
  pages={11809--11822},
  year={2023}
}

@article{uesato2022solving,
  title={Solving Math Word Problems with Process- and Outcome-Based Feedback},
  author={Uesato, Jonathan and Kushman, Nate and Kumar, Ramana and Song, Francis and Siegel, Noah and Wang, Lisa and Creswell, Antonia and Irving, Geoffrey and Higgins, Irina},
  journal={arXiv preprint arXiv:2211.14275},
  year={2022}
}

@inproceedings{zhang2025generative,
  title={Generative Verifiers: Reward Modeling as Next-Token Prediction},
  author={Zhang, Lunjun and Hosseini, Arian and Bansal, Hritik and Kazemi, Mehran and Kumar, Aviral and Agarwal, Rishabh},
  booktitle={International Conference on Learning Representations},
  year={2025}
}

@article{kim2026reliability,
  title={Reliability-Aware Adaptive Self-Consistency for Efficient Sampling in {LLM} Reasoning},
  author={Kim, Junseok and Yang, Nakyeong and Min, Kyungmin and Jung, Kyomin},
  journal={arXiv preprint arXiv:2601.02970},
  year={2026}
}

@article{ota2026cite,
  title={{CITE}: Anytime-Valid Statistical Inference in {LLM} Self-Consistency},
  author={Ota, Hirofumi and Iwase, Naoto and Ichihara, Yuki and Komiyama, Junpei and Imaizumi, Masaaki},
  journal={arXiv preprint arXiv:2605.05873},
  year={2026}
}

@article{huang2025efficient,
  title={Efficient Test-Time Scaling via Self-Calibration},
  author={Huang, Chengsong and Huang, Langlin and Leng, Jixuan and Liu, Jiacheng and Huang, Jiaxin},
  journal={arXiv preprint arXiv:2503.00031},
  year={2025}
}

@article{yang2025dynamic,
  title={Dynamic Early Exit in Reasoning Models},
  author={Yang, Chenxu and Si, Qingyi and Duan, Yongjie and Zhu, Zheliang and Zhu, Chenyu and Li, Qiaowei and Chen, Minghui and Lin, Zheng and Wang, Weiping},
  journal={arXiv preprint arXiv:2504.15895},
  year={2025}
}

@inproceedings{laaouach2025halt,
  title={{HALT-CoT}: Model-Agnostic Early Stopping for Chain-of-Thought Reasoning via Answer Entropy},
  author={Laaouach, Yassir},
  booktitle={4th Muslims in ML Workshop co-located with ICML 2025},
  url={https://openreview.net/forum?id=CX5c7C1CZa},
  year={2025}
}

@article{hou2025thinkprune,
  title={ThinkPrune: Pruning Long Chain-of-Thought of {LLMs} via Reinforcement Learning},
  author={Hou, Bairu and Zhang, Yang and Ji, Jiabao and Liu, Yujian and Qian, Kaizhi and Andreas, Jacob and Chang, Shiyu},
  journal={arXiv preprint arXiv:2504.01296},
  year={2025}
}

@article{lin2022teaching,
  title={Teaching Models to Express Their Uncertainty in Words},
  author={Lin, Stephanie and Hilton, Jacob and Evans, Owain},
  journal={Transactions on Machine Learning Research},
  year={2022}
}

@inproceedings{hendrycks2021measuring,
  title={Measuring Mathematical Problem Solving with the {MATH} Dataset},
  author={Hendrycks, Dan and Burns, Collin and Kadavath, Saurav and Arora, Akul and Basart, Steven and Tang, Eric and Song, Dawn and Steinhardt, Jacob},
  booktitle={Advances in Neural Information Processing Systems},
  volume={34},
  pages={15941--15952},
  year={2021}
}

\newpage
\appendix

\section{Proofs}
\label{app:proofs}

\subsection{Margin decomposition}

\begin{Proposition}[Margin decomposition]\label{thm:decomp}
Fix checkpoint $t$ and challenger $k\neq L(t)$. The full-budget margin splits as
\begin{equation}\label{eq:decomp}
M_k(T)=R_k(t)+\Phi_k(t),
\end{equation}
where $R_k(t)=\sum_{\substack{j:a_j(t)=L,\,a_j(T)=a_j(t)}}w_j - \sum_{\substack{j:a_j(t)=k,\,a_j(T)=a_j(t)}}w_j$ is the retained margin, and $\Phi_k(t)=\sum_{\substack{j:a_j(T)=L,\,a_j(t)\neq L}}w_j - \sum_{\substack{j:a_j(T)=k,\,a_j(t)\neq k}}w_j$ is the switch flow.
\end{Proposition}

\begin{proof}
Partition the final vote for $L$ into traces that already voted for $L$ at checkpoint $t$ and stayed there, plus traces that switched into $L$. Do the same for $k$. Subtracting the two final vote totals gives Eq.~\ref{eq:decomp}.
\end{proof}

Using switch indicators $X_j = \mathbf{1}\{a_j(T)\neq a_j(t)\}$, the retained margin becomes $R_k = M_k(t) - \sum_{j:a_j(t)=L} X_j w_j + \sum_{j:a_j(t)=k} X_j w_j$. Bounding $\Phi_k$ adversarially (every departing leader/neutral voter joins $k$): $\Phi_k \geq -\sum_{j:a_j(t)=L} X_j w_j - \sum_{j:a_j(t)\notin\{L,k\}} X_j w_j$. Summing $R_k + \Phi_k$ and collecting by current vote yields $M_k(T) \geq M_k(t) - \sum_{j\in\mathcal{A}_t} X_j \cdot c_j^k$, recovering Definition~\ref{def:cost}.

\subsection{Proof of Theorem~\ref{thm:safety}}

\begin{proof}
First consider an observed challenger $k\neq L(t)$. Let
\[
Y_j =
w_j\!\left(\mathbf{1}\{a_j(t)=L\}-\mathbf{1}\{a_j(t)=k\}\right)
-
w_j\!\left(\mathbf{1}\{a_j(T)=L\}-\mathbf{1}\{a_j(T)=k\}\right)
\]
be trace $j$'s realized decrease in the leader-versus-$k$ margin between checkpoint $t$ and the full-budget endpoint $T$. Conditional on $\mathcal{F}_t$, the variables $Y_j$ are independent by Assumption~\ref{ass:indep}. Each has range width at most $2w_{\max}$, since one trace can change the leader-versus-$k$ margin by at most $2w_j$. If $q_j(t)=0$, then $a_j(T)=a_j(t)$ almost surely and $Y_j=0$ almost surely, so only the $N_{\mathrm{active}}$ traces contribute nonzero range.

Definition~\ref{def:cost} upper-bounds the conditional expectation of each trace's adverse contribution. If $a_j(t)=L$, a switch can decrease the margin by at most $2w_j$, so $\mathbb{E}[Y_j\mid\mathcal{F}_t]\leq 2q_j(t)w_j$. If $a_j(t)=k$, any switch away from $k$ increases the margin by at least $w_j$, so $\mathbb{E}[Y_j\mid\mathcal{F}_t]\leq -q_j(t)w_j$. If $a_j(t)\notin\{L,k\}$, a switch can decrease the margin by at most $w_j$, so $\mathbb{E}[Y_j\mid\mathcal{F}_t]\leq q_j(t)w_j$. Therefore
\[
\mu=\mathbb{E}\!\left[\sum_jY_j\mid\mathcal{F}_t\right]\leq \sum_j q_j(t)c_j^k=\Gamma_k(t;1).
\]
Challenger $k$ overtakes the leader only if $\sum_jY_j\geq M_k(t)$. For $M_k(t)>\Gamma_k(t;1)$, the bound $\mu\leq\Gamma_k(t;1)$ implies that this event requires $\sum_j(Y_j-\mathbb{E}[Y_j\mid\mathcal{F}_t])\geq M_k(t)-\Gamma_k(t;1)$. Hoeffding's inequality gives
\[
P\!\left(\sum_jY_j\geq M_k(t)\mid\mathcal{F}_t\right)
\leq
\exp\!\left(
-\frac{(M_k(t)-\Gamma_k(t;1))^2}{2w_{\max}^2N_{\mathrm{active}}}
\right).
\]
When $M_k(t)\leq \Gamma_k(t;1)$, the same bound with the positive part is vacuous but valid.

For the synthetic unseen challenger $\bot$, define
\[
\widetilde Y_j =
2w_j\mathbf{1}\{a_j(t)=L,\ a_j(T)\neq a_j(t)\}
+w_j\mathbf{1}\{a_j(t)\neq L,\ a_j(T)\neq a_j(t)\},
\]
and set $M_{\bot}(T)=M_{\bot}(t)-\sum_j\widetilde Y_j$ with $M_{\bot}(t)=V_L(t)$. This synthetic endpoint margin is pessimistic for every answer string $b$ absent at checkpoint $t$: the actual decrease in the leader-versus-$b$ margin is at most $\sum_j\widetilde Y_j$. Hence, if any previously unseen answer overtakes the leader, then $M_{\bot}(T)\leq 0$.

The variables $\widetilde Y_j$ are conditionally independent, have range width at most $2w_{\max}$, and are almost surely zero when $q_j(t)=0$. They satisfy
\[
\mathbb{E}[\widetilde Y_j\mid\mathcal{F}_t]
=
\begin{cases}
2q_j(t)w_j, & a_j(t)=L,\\
q_j(t)w_j, & a_j(t)\neq L,
\end{cases}
=q_j(t)c_j^{\bot}.
\]
The same Hoeffding argument therefore applies to $\bot$ and gives Eq.~\ref{eq:safety}.
\end{proof}

\subsection{Proof of Corollary~\ref{cor:stopping}}

\begin{proof}
At checkpoint $t$, the augmented challenger set $\mathcal{K}(t)$ contains at most $N$ challengers: at most $N-1$ observed nonleader answers, plus the synthetic unseen challenger $\bot$. Eq.~\ref{eq:correction} makes the failure probability for each checked challenger at most $\delta/N$ by Theorem~\ref{thm:safety}. A union bound over $\mathcal{K}(t)$ gives
\[
P\bigl(L(T)\neq L(t)\mid\mathcal{F}_t\bigr)\leq\delta .
\]
The unseen challenger covers all answer labels absent at checkpoint $t$, so the union bound includes both observed and newly appearing challengers.

For the stopping-time statement, let the finite set of checkpoints be $\{t_1,\ldots,t_P\}$ and let $E_i=\{\tau=t_i\}$. Since $E_i\in\mathcal{F}_{t_i}$ and the conditional failure probability is at most $\delta$ whenever the rule stops,
\[
P(L(T)\neq L(\tau))
=
\sum_{i=1}^P
\mathbb{E}\!\left[
\mathbf{1}_{E_i}
P\!\left(L(T)\neq L(t_i)\mid\mathcal{F}_{t_i}\right)
\right]
\leq
\sum_{i=1}^P \delta P(E_i)
\leq \delta .
\]
\end{proof}

\section{From optimal stopping to MARS: derivation and modeling choices}
\label{app:optimality}

This appendix develops the full derivation connecting \method to the Bayes-optimal stopping rule, showing that \method is a strict conservative relaxation.

\subsection{From optimal stopping to margin certification}

The Bayes-optimal early-stopping rule halts at the earliest checkpoint $t$ satisfying
\[
P\bigl(L(T)\neq L(t)\mid \mathcal{F}_t\bigr)\leq \delta.
\]
This is intractable: it requires the joint future distribution of all trace answers. We reduce it to a per-challenger margin condition. Since the leader can only be overturned if some challenger $k$ achieves $M_k(T)\leq 0$, a union bound gives
\[
P(L(T)\neq L(t)\mid\mathcal{F}_t) \leq \sum_{k\in\mathcal{K}(t)} P(M_k(T)\leq 0\mid\mathcal{F}_t).
\]
It therefore suffices to certify $P(M_k(T)\leq 0\mid\mathcal{F}_t)\leq\delta/|\mathcal{K}(t)|$ for each challenger. The problem becomes: bound the probability that the margin against each challenger drops to zero.

\subsection{Margin decomposition}

The full-budget margin decomposes exhaustively (Proposition~\ref{thm:decomp}):
\[
M_k(T) = R_k(t) + \Phi_k(t),
\]
where $R_k$ is the retained margin from traces keeping their answer, and $\Phi_k$ is the switch flow from traces that change. Using switch indicators $X_j = \mathbf{1}\{a_j(T)\neq a_j(t)\}$:
\[
R_k = M_k(t) - \sum_{\substack{j\in\mathcal{A}_t:\\a_j(t)=L}} X_j\, w_j + \sum_{\substack{j\in\mathcal{A}_t:\\a_j(t)=k}} X_j\, w_j.
\]
This is the current margin $M_k(t)$, minus leader-voters who depart, plus challenger-voters who depart.

\subsection{Retained margin: learnable via logistic regression}

The retained term $R_k$ depends only on \emph{whether} each trace switches ($X_j$). The conditional switch probability $q_j(t) = P(X_j=1\mid\mathcal{F}_t)$ is a per-trace quantity depending on observable features: checkpoint position, probe confidence, answer-flip count, stability streak. These are all $\mathcal{F}_t$-measurable, making $q_j$ a standard binary prediction problem. A lightweight logistic regression trained on warmup traces estimates $q_j$ well (Appendix~\ref{app:qmodel}). This component of the margin can therefore be modeled tightly.

\subsection{Switch flow: unpredictable, therefore adversarial}

The switch-flow term $\Phi_k$ depends on \emph{where} switching traces land. The destination is an open-ended answer string determined by future reasoning, which is not learnable from $\mathcal{F}_t$ alone. We therefore bound $\Phi_k$ adversarially. In the worst case, every switcher leaving the leader or a neutral position joins challenger $k$:
\[
\Phi_k \geq -\sum_{\substack{j\in\mathcal{A}_t:\\a_j(t)=L}} X_j\, w_j - \sum_{\substack{j\in\mathcal{A}_t:\\a_j(t)\notin\{L,k\}}} X_j\, w_j.
\]
Summing $R_k + \Phi_k$ and collecting terms by current vote:
\[
M_k(T) \geq M_k(t) - \sum_{j\in\mathcal{A}_t} X_j \cdot c_j^k,
\]
where $c_j^k$ is the adversarial switch cost (Definition~\ref{def:cost}). This is a one-sided bound: the true $M_k(T)$ is at least as large as this expression.

\subsection{Conservative relaxation: formal statement}

Applying Hoeffding's inequality to $\sum X_j c_j^k$ (Theorem~\ref{thm:safety}) and a union bound over challengers (Corollary~\ref{cor:stopping}) yields the stopping rule: $M_k(t)\geq\Gamma_k(t)+\epsilon(N,\delta)$ for all $k$. This is a \emph{sufficient} condition for $P(L(T)\neq L(t)\mid\mathcal{F}_t)\leq\delta$.

\begin{Proposition}[Conservative relaxation]\label{prop:conservative}
Let $\tau^*$ be the Bayes-optimal stopping time and $\tau_{\mathrm{MARS}}$ the stopping time of the certified \method rule (Eq.~\ref{eq:stopping} with $\gamma=1$, true $q_j$). Under Assumptions~\ref{ass:bounded}--\ref{ass:indep}, $\tau_{\mathrm{MARS}}\geq\tau^*$ almost surely.
\end{Proposition}

\begin{proof}
By Corollary~\ref{cor:stopping}, whenever Eq.~\ref{eq:stopping} holds for all $k\in\mathcal{K}(t)$, we have $P(L(T)\neq L(t)\mid\mathcal{F}_t)\leq\delta$. The checkpoints where \method fires are therefore a subset of those where the Bayes-optimal condition holds. Since both rules stop at the earliest qualifying checkpoint, $\tau_{\mathrm{MARS}}\geq\tau^*$.
\end{proof}

\section{Switch probability model details}
\label{app:qmodel}

This appendix details the logistic regression model used to estimate per-trace switch probabilities $q_j(t) = P(a_j(T) \neq a_j(t) \mid \mathcal{F}_t)$.

\subsection{Features}

The model uses 5 features, all $\mathcal{F}_t$-measurable (observable at checkpoint $t$ without future information):

\begin{table}[h]
\centering
\small
\begin{tabular}{clll}
\toprule
\# & Feature & Type & Description \\
\midrule
1 & \texttt{position} & int & Absolute probe position in tokens (e.g., 2048, 4096, \ldots) \\
2 & \texttt{confidence} & float & Probe confidence score at position $t$ \\
3 & \texttt{flips} & int & Cumulative count of answer changes across probes $[0, \ldots, t]$ \\
4 & \texttt{streak} & int & Consecutive same-answer count ending at position $t$ (min 1) \\
5 & \texttt{conf\_trend} & float & $\texttt{confidence}(t) - \texttt{confidence}(t-1)$; 0 at first position \\
\bottomrule
\end{tabular}
\caption{Features for the switch probability model. All features are trace-intrinsic: they describe the trace's own history, not its relationship to the vote distribution.}
\label{tab:features}
\end{table}

\paragraph{Feature intuition.}
\texttt{position} captures the baseline switch rate, which decreases as traces approach completion.
\texttt{confidence} reflects the model's certainty in its current answer; prior work shows that model confidence and expressed uncertainty can carry information about answer reliability~\citep{kadavath2022language,lin2022teaching}, and in our data high-confidence probes switch less often.
\texttt{flips} measures cumulative instability; traces that have changed answers many times are more likely to change again.
\texttt{streak} measures recent stability; a trace holding the same answer for many consecutive probes is unlikely to switch.
\texttt{conf\_trend} captures momentum: rising confidence suggests convergence, falling confidence suggests the trace may switch.

\subsection{Training}

\paragraph{Data.} For each question, the model is trained on 16 warmup traces that are run to completion, providing both intermediate answers at each checkpoint and the final answer $a_j(T)$. Each (trace, position) pair yields one training example, giving approximately $16 \times 27 \approx 430$ examples per question. The binary label is $y_{j,t} = \mathbf{1}\{a_j(t) \neq a_j(T)\}$.

\paragraph{Standardization.} Features are standardized to zero mean and unit variance before fitting. The standardization parameters (mean, std per feature) are stored and applied at prediction time.

\paragraph{Logistic regression.} We fit:
\[
P(y = 1 \mid \mathbf{x}) = \sigma(\beta_0 + \boldsymbol{\beta}^\top \mathbf{x}_{\text{std}})
\]
where $\sigma(z) = 1/(1 + e^{-z})$ and $\mathbf{x}_{\text{std}}$ is the standardized feature vector. The parameters $\beta_0 \in \mathbb{R}$ and $\boldsymbol{\beta} \in \mathbb{R}^{5}$ (one coefficient per feature) are fit by minimizing the regularized negative log-likelihood:
\[
\mathcal{L}(\beta_0, \boldsymbol{\beta}) = -\sum_{i} \bigl[y_i \log \hat{p}_i + (1 - y_i) \log(1 - \hat{p}_i)\bigr] + \lambda \|\boldsymbol{\beta}\|_2^2
\]
with $\lambda = 0.01$ (mild L2 regularization on coefficients, not intercept). Optimization uses L-BFGS-B with a maximum of 200 iterations, initialized at $\boldsymbol{0}$.

\subsection{Platt calibration}

After fitting the logistic model, we apply Platt scaling \citep{platt1999probabilistic} to correct any systematic miscalibration. Given raw predictions $\hat{q}_j$ on the training set:
\[
q_j^{\text{cal}} = \sigma\!\left(a \cdot \text{logit}(\hat{q}_j) + b\right)
\]
where $\text{logit}(p) = \log(p / (1-p))$ and $(a, b)$ are fit by minimizing cross-entropy on the training predictions versus true labels. With a well-specified logistic model, this is approximately the identity ($a \approx 1, b \approx 0$).

\paragraph{Why Platt calibration.} The logistic model is trained with L2 regularization, which shrinks predictions toward 0.5. Platt scaling can undo this shrinkage. More importantly, if the 5 features are insufficient to capture all switch-relevant variation (which they are---they miss vote-context information), the raw predictions may be systematically miscalibrated. Platt scaling provides a lightweight correction.

\subsection{Prediction}

At prediction time, the fitted model produces $q_j(t)$ for every trace $j$ at every checkpoint $t$. The predictions are precomputed as a $[N_{\text{traces}} \times P]$ matrix and indexed into bootstrap samples. No per-iteration computation is needed---the model is fit once per question during the warmup phase.

\paragraph{Computational cost.} Fitting the model (L-BFGS on $\sim$430 examples with 6 parameters) takes $<$1ms per question. Prediction (matrix multiply + sigmoid) for all traces and positions takes $<$1ms. The $q$-model adds negligible overhead to the simulation.

\section{Trace generation and answer probing}
\label{app:trace-generation}

\paragraph{Trace generation.}
We generate traces using SGLang~\citep{sglang}. Table~\ref{tab:inference-params} lists the inference parameters for each model. We sample 4{,}096 traces per problem, all launched in parallel.

\begin{table}[h]
\centering
\caption{Inference parameters for trace generation.}
\label{tab:inference-params}
\begin{tabular}{lcccc}
\toprule
Model & Temperature & Top-$p$ & Top-$k$ & Max Seq.\ Length \\
\midrule
DeepSeek-R1-Distill-Qwen-8B & 0.6 & 0.95 & --- & 64k \\
Qwen3-32B & 0.6 & 0.95 & 20 & 32k \\
Qwen3-Next & 0.6 & 0.95 & 20 & 82k \\
\bottomrule
\end{tabular}
\end{table}

\paragraph{Answer probing at checkpoints.}
Checkpoints are spaced every 2{,}048 tokens. At checkpoint~$t$, we truncate the reasoning trace at position~$t$ and append:

\begin{quote}
\texttt{Considering the limited time by the user, I have to give the solution based on the thinking directly now.\textbackslash n</think>\textbackslash n\textbackslash boxed\{}
\end{quote}

\noindent This closes the thinking block and forces the model to commit to an answer. We allow up to 50 new tokens for this completion; temperature, top-$p$, and top-$k$ stay the same as in trace generation. The answer is parsed from the resulting \verb|\boxed{}| output.

\section{Full experimental results}
\label{app:full_results}

Tables~\ref{tab:full_ds8b},~\ref{tab:full_q32b}, and~\ref{tab:full_qnext} report the complete results for all methods across all model--dataset combinations (3 models $\times$ 3 datasets). Each table shows \method at $\gamma = 1$ (fully conservative, no calibration) and with per-question $\gamma$-calibration, for both the learned $q$-model and oracle $q$, on SC and DCO. Savings are mean per-question token reductions relative to each family's own baseline (SC or DCO).

\begin{table}[h]
\centering
\footnotesize
\setlength{\tabcolsep}{3.5pt}
\caption{Full results for DeepSeek-R1-Distill-Qwen-8B.}
\label{tab:full_ds8b}
\begin{tabular}{llccc}
\toprule
& Method & Accuracy & Mean Tokens & Savings \\
\midrule
    \multirow{10}{*}{\rotatebox{90}{\scriptsize BRUMO '25}} & SC (baseline) & 93.2\% & 11{,}795{,}919 & --- \\
     & \method ($\gamma{=}1$) & 93.2\% & 9{,}246{,}472 & 28.8\% \\
     & \method & 93.2\% & 8{,}809{,}642 & 33.7\% \\
     & \method (oracle $q$, $\gamma{=}1$) & 93.2\% & 8{,}856{,}451 & 33.6\% \\
     & \method (oracle $q$) & 93.2\% & 8{,}419{,}773 & 37.9\% \\
\cmidrule{2-5}
     & DCO (baseline) & 93.3\% & 4{,}908{,}579 & --- \\
     & DCO + \method ($\gamma{=}1$) & 93.3\% & 4{,}462{,}553 & 13.7\% \\
     & DCO + \method & 93.3\% & 4{,}366{,}515 & 17.1\% \\
     & DCO + \method (oracle $q$, $\gamma{=}1$) & 93.3\% & 4{,}323{,}560 & 18.2\% \\
     & DCO + \method (oracle $q$) & 93.3\% & 4{,}287{,}402 & 19.3\% \\
\midrule
    \multirow{10}{*}{\rotatebox{90}{\scriptsize AIME '25}} & SC (baseline) & 83.3\% & 13{,}210{,}697 & --- \\
     & \method ($\gamma{=}1$) & 83.3\% & 10{,}464{,}233 & 31.1\% \\
     & \method & 83.3\% & 9{,}779{,}804 & 36.3\% \\
     & \method (oracle $q$, $\gamma{=}1$) & 83.3\% & 10{,}072{,}036 & 34.3\% \\
     & \method (oracle $q$) & 83.3\% & 9{,}444{,}399 & 38.7\% \\
\cmidrule{2-5}
     & DCO (baseline) & 88.6\% & 5{,}782{,}386 & --- \\
     & DCO + \method ($\gamma{=}1$) & 88.8\% & 5{,}256{,}042 & 15.8\% \\
     & DCO + \method & 88.6\% & 5{,}136{,}098 & 18.2\% \\
     & DCO + \method (oracle $q$, $\gamma{=}1$) & 88.7\% & 5{,}144{,}903 & 18.4\% \\
     & DCO + \method (oracle $q$) & 88.6\% & 5{,}045{,}326 & 20.5\% \\
\midrule
    \multirow{10}{*}{\rotatebox{90}{\scriptsize HMMT}} & SC (baseline) & 70.0\% & 14{,}882{,}765 & --- \\
     & \method ($\gamma{=}1$) & 70.0\% & 12{,}491{,}934 & 22.7\% \\
     & \method & 70.0\% & 11{,}404{,}709 & 29.3\% \\
     & \method (oracle $q$, $\gamma{=}1$) & 70.0\% & 12{,}139{,}649 & 25.7\% \\
     & \method (oracle $q$) & 70.0\% & 11{,}205{,}835 & 31.5\% \\
\cmidrule{2-5}
     & DCO (baseline) & 78.7\% & 5{,}956{,}553 & --- \\
     & DCO + \method ($\gamma{=}1$) & 78.2\% & 5{,}442{,}712 & 11.7\% \\
     & DCO + \method & 78.6\% & 5{,}367{,}318 & 13.8\% \\
     & DCO + \method (oracle $q$, $\gamma{=}1$) & 78.6\% & 5{,}354{,}571 & 13.9\% \\
     & DCO + \method (oracle $q$) & 78.6\% & 5{,}327{,}150 & 14.7\% \\
\bottomrule
\end{tabular}
\end{table}

\begin{table}[h]
\centering
\footnotesize
\setlength{\tabcolsep}{3.5pt}
\caption{Full results for Qwen3-32B.}
\label{tab:full_q32b}
\begin{tabular}{llccc}
\toprule
& Method & Accuracy & Mean Tokens & Savings \\
\midrule
    \multirow{10}{*}{\rotatebox{90}{\scriptsize BRUMO '25}} & SC (baseline) & 90.9\% & 7{,}138{,}843 & --- \\
     & \method ($\gamma{=}1$) & 90.6\% & 5{,}437{,}125 & 31.0\% \\
     & \method & 90.3\% & 5{,}013{,}437 & 36.2\% \\
     & \method (oracle $q$, $\gamma{=}1$) & 90.9\% & 5{,}292{,}775 & 33.6\% \\
     & \method (oracle $q$) & 90.9\% & 4{,}979{,}397 & 37.1\% \\
\cmidrule{2-5}
     & DCO (baseline) & 92.4\% & 4{,}627{,}106 & --- \\
     & DCO + \method ($\gamma{=}1$) & 92.9\% & 3{,}894{,}126 & 21.2\% \\
     & DCO + \method & 93.0\% & 3{,}801{,}181 & 23.2\% \\
     & DCO + \method (oracle $q$, $\gamma{=}1$) & 92.5\% & 3{,}854{,}197 & 22.2\% \\
     & DCO + \method (oracle $q$) & 92.5\% & 3{,}749{,}830 & 24.7\% \\
\midrule
    \multirow{10}{*}{\rotatebox{90}{\scriptsize AIME '25}} & SC (baseline) & 80.1\% & 7{,}739{,}219 & --- \\
     & \method ($\gamma{=}1$) & 80.1\% & 5{,}674{,}397 & 31.2\% \\
     & \method & 80.1\% & 5{,}260{,}623 & 37.1\% \\
     & \method (oracle $q$, $\gamma{=}1$) & 80.1\% & 5{,}462{,}107 & 34.6\% \\
     & \method (oracle $q$) & 80.1\% & 5{,}069{,}577 & 40.1\% \\
\cmidrule{2-5}
     & DCO (baseline) & 79.8\% & 4{,}832{,}386 & --- \\
     & DCO + \method ($\gamma{=}1$) & 80.6\% & 3{,}834{,}011 & 21.4\% \\
     & DCO + \method & 80.6\% & 3{,}661{,}690 & 26.0\% \\
     & DCO + \method (oracle $q$, $\gamma{=}1$) & 80.2\% & 3{,}702{,}460 & 24.7\% \\
     & DCO + \method (oracle $q$) & 80.1\% & 3{,}513{,}299 & 29.4\% \\
\midrule
    \multirow{10}{*}{\rotatebox{90}{\scriptsize HMMT}} & SC (baseline) & 62.8\% & 8{,}943{,}718 & --- \\
     & \method ($\gamma{=}1$) & 62.8\% & 7{,}637{,}955 & 20.6\% \\
     & \method & 62.8\% & 7{,}293{,}460 & 24.7\% \\
     & \method (oracle $q$, $\gamma{=}1$) & 62.8\% & 7{,}525{,}361 & 22.3\% \\
     & \method (oracle $q$) & 62.8\% & 7{,}176{,}939 & 26.4\% \\
\cmidrule{2-5}
     & DCO (baseline) & 65.1\% & 5{,}636{,}835 & --- \\
     & DCO + \method ($\gamma{=}1$) & 65.5\% & 5{,}158{,}842 & 12.7\% \\
     & DCO + \method & 65.5\% & 5{,}030{,}029 & 15.7\% \\
     & DCO + \method (oracle $q$, $\gamma{=}1$) & 65.5\% & 5{,}125{,}956 & 13.6\% \\
     & DCO + \method (oracle $q$) & 65.7\% & 4{,}967{,}311 & 17.0\% \\
\bottomrule
\end{tabular}
\end{table}

\begin{table}[h]
\centering
\footnotesize
\setlength{\tabcolsep}{3.5pt}
\caption{Full results for Qwen3-next.}
\label{tab:full_qnext}
\begin{tabular}{llccc}
\toprule
& Method & Accuracy & Mean Tokens & Savings \\
\midrule
    \multirow{10}{*}{\rotatebox{90}{\scriptsize BRUMO '25}} & SC (baseline) & 96.7\% & 11{,}116{,}671 & --- \\
     & \method ($\gamma{=}1$) & 96.7\% & 6{,}938{,}809 & 42.9\% \\
     & \method & 96.7\% & 6{,}016{,}275 & 47.4\% \\
     & \method (oracle $q$, $\gamma{=}1$) & 96.7\% & 6{,}145{,}701 & 47.5\% \\
     & \method (oracle $q$) & 96.7\% & 5{,}701{,}379 & 50.2\% \\
\cmidrule{2-5}
     & DCO (baseline) & 95.4\% & 4{,}665{,}047 & --- \\
     & DCO + \method ($\gamma{=}1$) & 95.4\% & 3{,}593{,}620 & 24.7\% \\
     & DCO + \method & 95.5\% & 3{,}072{,}580 & 29.4\% \\
     & DCO + \method (oracle $q$, $\gamma{=}1$) & 95.5\% & 3{,}240{,}662 & 28.2\% \\
     & DCO + \method (oracle $q$) & 95.5\% & 2{,}996{,}557 & 30.9\% \\
\midrule
    \multirow{10}{*}{\rotatebox{90}{\scriptsize AIME '25}} & SC (baseline) & 86.9\% & 11{,}757{,}443 & --- \\
     & \method ($\gamma{=}1$) & 87.2\% & 8{,}662{,}693 & 35.1\% \\
     & \method & 87.7\% & 7{,}899{,}797 & 41.4\% \\
     & \method (oracle $q$, $\gamma{=}1$) & 86.9\% & 8{,}257{,}982 & 39.4\% \\
     & \method (oracle $q$) & 87.0\% & 7{,}559{,}786 & 44.4\% \\
\cmidrule{2-5}
     & DCO (baseline) & 90.0\% & 4{,}859{,}163 & --- \\
     & DCO + \method ($\gamma{=}1$) & 90.0\% & 4{,}188{,}131 & 16.6\% \\
     & DCO + \method & 90.0\% & 3{,}998{,}876 & 21.2\% \\
     & DCO + \method (oracle $q$, $\gamma{=}1$) & 90.0\% & 4{,}053{,}656 & 20.5\% \\
     & DCO + \method (oracle $q$) & 90.0\% & 3{,}906{,}587 & 23.4\% \\
\midrule
    \multirow{10}{*}{\rotatebox{90}{\scriptsize HMMT}} & SC (baseline) & 86.9\% & 15{,}432{,}157 & --- \\
     & \method ($\gamma{=}1$) & 86.9\% & 12{,}884{,}039 & 26.7\% \\
     & \method & 86.9\% & 12{,}280{,}735 & 31.4\% \\
     & \method (oracle $q$, $\gamma{=}1$) & 86.9\% & 12{,}401{,}908 & 30.9\% \\
     & \method (oracle $q$) & 86.9\% & 11{,}842{,}803 & 34.4\% \\
\cmidrule{2-5}
     & DCO (baseline) & 82.2\% & 6{,}960{,}432 & --- \\
     & DCO + \method ($\gamma{=}1$) & 82.2\% & 6{,}428{,}595 & 12.9\% \\
     & DCO + \method & 82.0\% & 6{,}290{,}936 & 16.1\% \\
     & DCO + \method (oracle $q$, $\gamma{=}1$) & 82.2\% & 6{,}283{,}393 & 16.5\% \\
     & DCO + \method (oracle $q$) & 82.2\% & 6{,}203{,}829 & 17.8\% \\
\bottomrule
\end{tabular}
\end{table}

\section{Limitations \& Broader impacts}
\label{app:limitations}
\paragraph{Limitations.}

The margin decomposition assumes constant per-trace weights $w_j$, which holds exactly for uniform-weight schemes (self-consistency) but only approximately for methods with position-dependent weighting such as DCO. In practice, this approximation is benign: DCO's own filtering and truncation absorb most of the weight variation, and our method preserves DCO baseline accuracy across all 9 settings tested. Our evaluation covers mathematical reasoning; the framework's applicability to other domains (e.g., coding) remains to be validated.

\paragraph{Broader impacts.}
The main positive impact is reducing the compute cost, latency, and energy use of parallel LLM inference. The same efficiency gain could also make high-volume LLM deployment cheaper, including deployments whose outputs are low-quality or harmful. \method does not introduce a new model or new generation capability, but it should be deployed with the same output-quality and safety controls as the underlying LLM system.


\end{document}